\documentclass{article}
\usepackage[latin9]{inputenc}
\usepackage{array}
\usepackage{float}
\usepackage{wrapfig}
\usepackage{units}
\usepackage{multirow}
\usepackage{amsmath}
\usepackage{amsthm}
\usepackage{graphicx}
\usepackage[unicode=true,
 bookmarks=false,
 breaklinks=false,pdfborder={0 0 1},backref=section,colorlinks=false]
 {hyperref}

\makeatletter

\providecommand{\tabularnewline}{\\}
\floatstyle{ruled}
\newfloat{algorithm}{tbp}{loa}
\providecommand{\algorithmname}{Algorithm}
\floatname{algorithm}{\protect\algorithmname}

\theoremstyle{plain}
\newtheorem{thm}{\protect\theoremname}
  \theoremstyle{definition}
  \newtheorem{defn}[thm]{\protect\definitionname}
  \theoremstyle{plain}
  \newtheorem{lem}[thm]{\protect\lemmaname}
  \theoremstyle{plain}
  \newtheorem{cor}[thm]{\protect\corollaryname}
  \theoremstyle{plain}
  \newtheorem{prop}[thm]{\protect\propositionname}

\usepackage{iclr2017_conference}
\usepackage{times}\usepackage{url}

\usepackage{algpseudocode}
\usepackage{amssymb}
\usepackage{nicefrac}

\@ifundefined{showcaptionsetup}{}{%
 \PassOptionsToPackage{caption=false}{subfig}}
\usepackage{subfig}
\makeatother

  \providecommand{\corollaryname}{Corollary}
  \providecommand{\definitionname}{Definition}
  \providecommand{\lemmaname}{Lemma}
  \providecommand{\propositionname}{Proposition}
\providecommand{\theoremname}{Theorem}

\begin{document}
\author{Tao Wei\footnotemark[2] \\
University at Buffalo \\
Buffalo, NY 14260 \\
\texttt{taowei@buffalo.edu} \\
\And
Changhu Wang \\
Microsoft Research \\
Beijing, China, 100080 \\
\texttt{chw@microsoft.com} \\
\And
Chang Wen Chen \\
University at Buffalo \\
Buffalo, NY 14260 \\
\texttt{chencw@buffalo.edu}}

\title{Modularized Morphing of Neural Networks}
\maketitle
\begin{abstract}
In this work we study the problem of network morphism, an effective
learning scheme to morph a well-trained neural network to a new one
with the network function completely preserved. Different from existing
work where basic morphing types on the layer level were addressed,
we target at the central problem of network morphism at a higher level,
i.e., \textit{how a convolutional layer can be morphed into an arbitrary
module of a neural network}. To simplify the representation of a network,
we abstract a module as a graph with blobs as vertices and convolutional
layers as edges, based on which the morphing process is able to be
formulated as a graph transformation problem. Two atomic morphing
operations are introduced to compose the graphs, based on which modules
are classified into two families, i.e., simple morphable modules and
complex modules. We present practical morphing solutions for both
of these two families, and prove that any reasonable module can be
morphed from a single convolutional layer. Extensive experiments have
been conducted based on the state-of-the-art ResNet on benchmark datasets,
and the effectiveness of the proposed solution has been verified.
\end{abstract}

\section{Introduction}

\renewcommand{\thefootnote}{\fnsymbol{footnote}}
\footnotetext[2]{Tao Wei performed this work while being an intern at Microsoft Research Asia.}
\renewcommand{\thefootnote}{\arabic{footnote}}

\renewcommand{\cite}{\citep}

Deep convolutional neural networks have continuously demonstrated
their excellent performances on diverse computer vision problems.
In image classification, the milestones of such networks can be roughly
represented by LeNet \cite{LeCunBoserDenkerEtAl1989}, AlexNet \cite{KrizhevskySutskeverHinton2012},
VGG net \cite{SimonyanZisserman2014}, GoogLeNet \cite{SzegedyLiuJiaEtAl2014},
and ResNet \cite{HeZhangRenEtAl2015}, with networks becoming deeper
and deeper. However, the architectures of these network are significantly
altered and hence are not backward-compatible. Considering a life-long
learning system, it is highly desired that the system is able to update
itself from the original version established initially, and then evolve
into a more powerful one, rather than re-learning a brand new one
from scratch.

Network morphism \cite{WeiWangRuiEtAl2016} is an effective way towards
such an ambitious goal. It can morph a well-trained network to a new
one with the knowledge entirely inherited, and hence is able to update
the original system to a compatible and more powerful one based on
further training. Network morphism is also a performance booster and
architecture explorer for convolutional neural networks, allowing
us to quickly investigate new models with significantly less computational
and human resources. However, the network morphism operations proposed
in \cite{WeiWangRuiEtAl2016}, including depth, width, and kernel
size changes, are quite primitive and have been limited to the level
of layer in a network. For practical applications where neural networks
usually consist of dozens or even hundreds of layers, the morphing
space would be too large for researchers to practically design the
architectures of target morphed networks, when based on these primitive
morphing operations only.

Different from previous work, we investigate in this research the
network morphism from a higher level of viewpoint, and systematically
study the central problem of network morphism on the module level,
i.e., \emph{whether and how a convolutional layer can be morphed into
an arbitrary module}\footnote{Although network morphism generally does not impose constraints on
the architecture of the child network, in this work we limit the investigation
to the expanding mode.}, where a module refers to a single-source, single-sink acyclic subnet
of a neural network. With this modularized network morphing, instead
of morphing in the layer level where numerous variations exist in
a deep neural network, we focus on the changes of basic modules of
networks, and explore the morphing space in a more efficient way.
The necessities for this study are two folds. First, we wish to explore
the capability of the network morphism operations and obtain a theoretical
upper bound for what we are able to do with this learning scheme.
Second, modern state-of-the-art convolutional neural networks have
been developed with modularized architectures \cite{SzegedyLiuJiaEtAl2014,HeZhangRenEtAl2015},
which stack the construction units following the same module design.
It is highly desired that the morphing operations could be directly
applied to these networks.

To study the morphing capability of network morphism and figure out
the morphing process, we introduce a simplified graph-based representation
for a module. Thus, the network morphing process can be formulated
as a graph transformation process. In this representation, the module
of a neural network is abstracted as a directed acyclic graph (DAG),
with data blobs in the network represented as vertices and convolutional
layers as edges. Furthermore, a vertex with more than one outdegree
(or indegree) implicitly includes a split of multiple copies of blobs
(or a joint of addition). Indeed, the proposed graph abstraction suffers
from the problem of dimension compatibility of blobs, for different
kernel filters may result in totally different blob dimensions. We
solve this problem by extending the blob and filter dimensions from
finite to infinite, and the convergence properties will also be carefully
investigated.

Two atomic morphing operations are adopted as the basis for the proposed
graph transformation, based on which a large family of modules can
be transformed from a convolutional layer. This family of modules
are called simple morphable modules in this work. A novel algorithm
is proposed to identify the morphing steps by reducing the module
into a single convolutional layer. For any module outside the simple
morphable family, i.e., complex module, we first apply the same reduction
process and reduce it to an irreducible module. A practical algorithm
is then proposed to solve for the network morphism equation of the
irreducible module. Therefore, we not only verify the morphability
to an arbitrary module, but also provide a unified morphing solution.
This demonstrates the generalization ability and thus practicality
of this learning scheme.

Extensive experiments have been conducted based on ResNet \cite{HeZhangRenEtAl2015}
to show the effectiveness of the proposed morphing solution. With
only 1.2x or less computation, the morphed network can achieve up
to 25\% relative performance improvement over the original ResNet.
Such an improvement is significant in the sense that the morphed 20-layered
network is able to achieve an error rate of 6.60\% which is even better
than a 110-layered ResNet (6.61\%) on the CIFAR10 dataset, with only
around $\nicefrac{1}{5}$ of the computational cost. It is also exciting
that the morphed 56-layered network is able to achieve 5.37\% error
rate, which is even lower than those of ResNet-110 (6.61\%) and ResNet-164
(5.46\%). The effectiveness of the proposed learning scheme has also
been verified on the CIFAR100 and ImageNet datasets.

\section{Related Work}

\emph{Knowledge Transfer}. Network morphism originated from knowledge
transferring for convolutional neural networks. Early attempts were
only able to transfer partial knowledge of a well-trained network.
For example, a series of model compression techniques \cite{BuciluCaruanaNiculescu-Mizil2006,BaCaruana2014,HintonVinyalsDean2015,RomeroBallasKahouEtAl2014}
were proposed to fit a lighter network to predict the output of a
heavier network. Pre-training \cite{SimonyanZisserman2014} was adopted
to pre-initialize certain layers of a deeper network with weights
learned from a shallower network. However, network morphism requires
the knowledge being fully transferred, and existing work includes
Net2Net \cite{ChenGoodfellowShlens2015} and NetMorph \cite{WeiWangRuiEtAl2016}.
Net2Net achieved this goal by padding identity mapping layers into
the neural network, while NetMorph decomposed a convolutional layer
into two layers by deconvolution. Note that the network morphism operations
in \cite{ChenGoodfellowShlens2015,WeiWangRuiEtAl2016} are quite primitive
and at a micro-scale layer level. In this research, we study the network
morphism at a meso-scale module level, and in particular, we investigate
its morphing capability.

\emph{Modularized Network Architecture}. The evolution of convolutional
neural networks has been from sequential to modularized. For example,
LeNet \cite{LeCunBottouBengioEtAl1998}, AlexNet \cite{KrizhevskySutskeverHinton2012},
and VGG net \cite{SimonyanZisserman2014} are sequential networks,
and their difference is primarily on the number of layers, which is
5, 8, and up to 19 respectively. However, recently proposed networks,
such as GoogLeNet \cite{SzegedyLiuJiaEtAl2014,SzegedyVanhouckeIoffeEtAl2015}
and ResNet \cite{HeZhangRenEtAl2015}, follow a modularized architecture
design, and have achieved the state-of-the-art performance. This is
why we wish to study network morphism at the module level, so that
its operations are able to directly apply to these modularized network
architectures.

\section{Network Morphism via Graph Abstraction}

In this section, we present a systematic study on the capability of
network morphism learning scheme. We shall verify that a convolutional
layer is able to be morphed into any single-source, single-sink DAG
subnet, named as a module here. We shall also present the corresponding
morphing algorithms. 

For simplicity, we first consider convolutional neural networks with
only convolutional layers. All other layers, including the non-linearity
and batch normalization layers, will be discussed later in this paper. 

\begin{figure}
\begin{centering}
\subfloat[Network morphism in depth. \label{fig:netmorph}]{\begin{centering}
\includegraphics[height=1.3in]{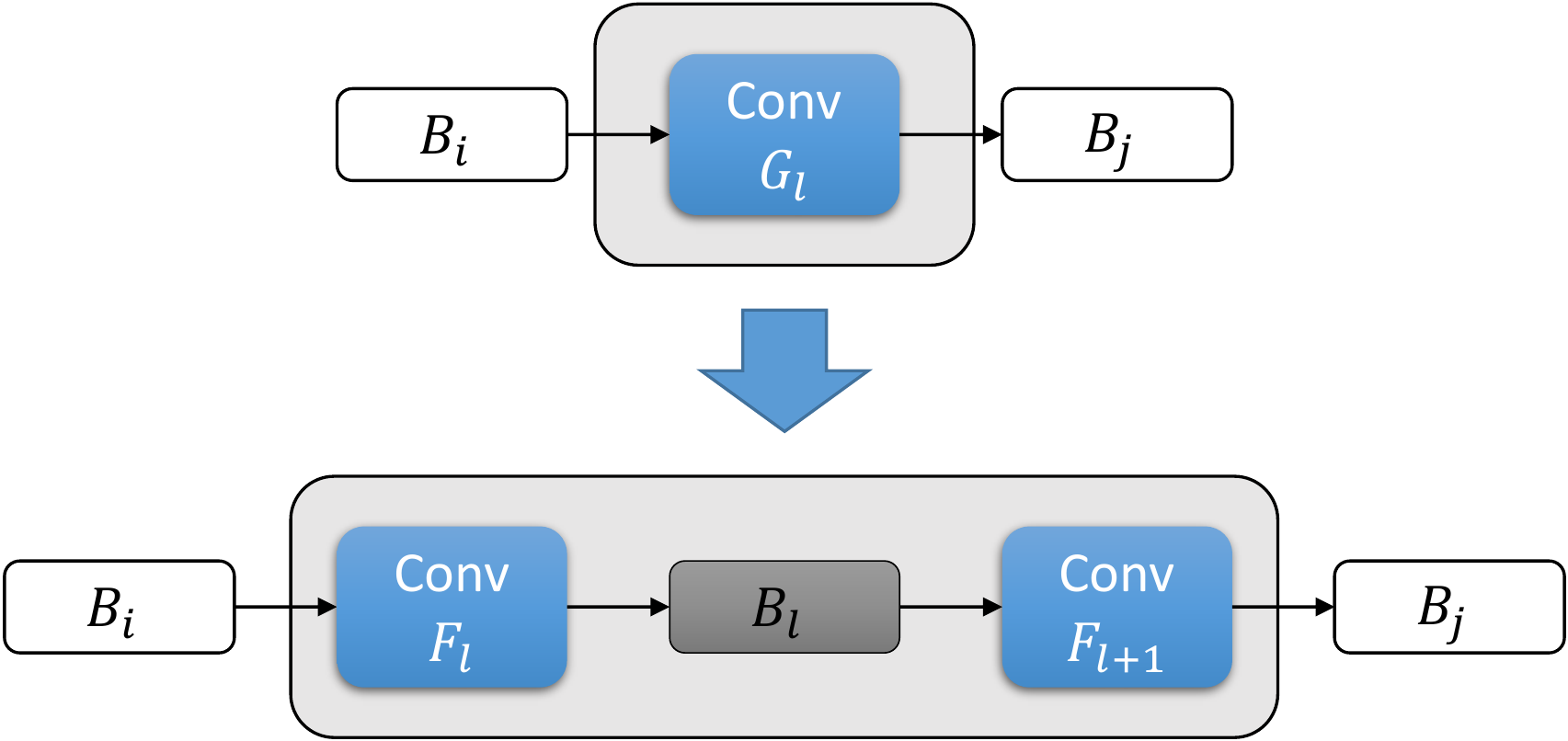} 
\par\end{centering}

}\quad{}\subfloat[Atomic morphing types. \label{fig:netmorph_atomic}]{\begin{centering}
\includegraphics[height=1.3in]{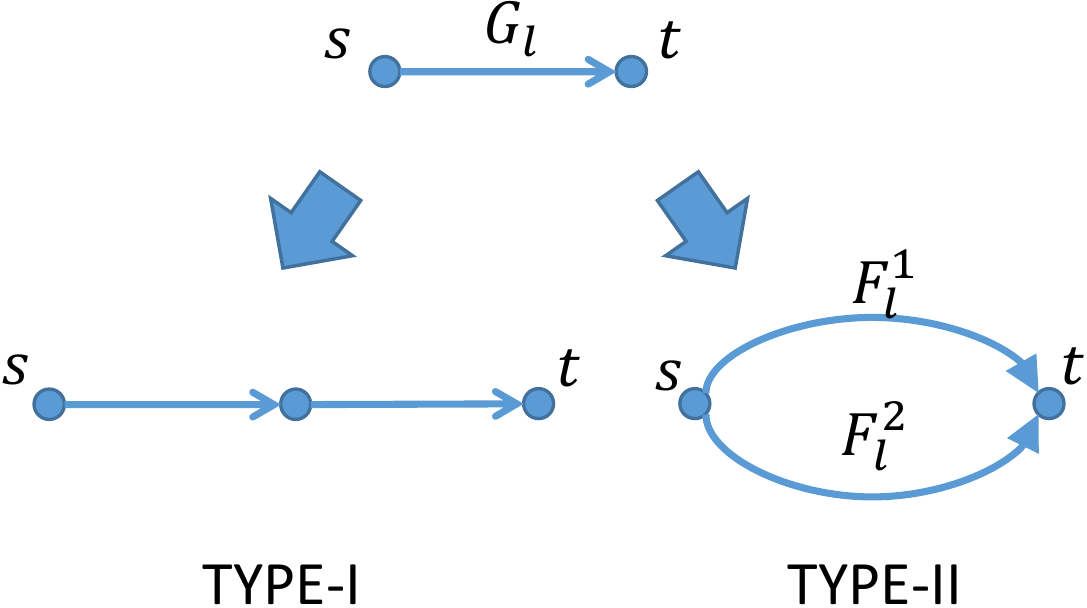} 
\par\end{centering}

}
\par\end{centering}

\caption{Illustration of atomic morphing types. (a) One convolutional layer
is morphed into two convolutional layers; (b) \texttt{TYPE-I} and
\texttt{TYPE-II} atomic morphing types.}
\end{figure}

\subsection{Background and Basic Notations}

For a 2D deep convolutional neural network (DCNN), as shown in Fig.
\ref{fig:netmorph}, the convolution is defined by: 
\begin{equation}
B_{j}(c_{j})=\sum_{c_{i}}B_{i}(c_{i})*G_{l}(c_{j},c_{i}),\label{eq:conv}
\end{equation}
where the blob $B_{*}$ is a 3D tensor of shape $(C_{*},H_{*},W_{*})$
and the convolutional filter $G_{l}$ is a 4D tensor of shape $(C_{j},C_{i},K_{l},K_{l}).$
In addition, $C_{*}$, $H_{*}$, and $W_{*}$ represent the number
of channels, height and width of $B_{*}$. $K_{l}$ is the convolutional
kernel size\footnote{Generally speaking, $G_{l}$ is a 4D tensor of shape $(C_{j},C_{i},K_{l}^{H},K_{l}^{W})$,
where convolutional kernel sizes for blob height and width are not
necessary to be the same. However, in order to simplify the notations,
we assume that $K_{l}^{H}=K_{l}^{W}$, but the claims and theorems
in this paper apply equally well when they are different.}.

In a network morphism process, the convolutional layer $G_{l}$ in
the parent network is morphed into two convolutional layers $F_{l}$
and $F_{l+1}$ (Fig. \ref{fig:netmorph}), where the filters $F_{l}$
and $F_{l+1}$ are 4D tensors of shapes $(C_{l},C_{i},K_{1},K_{1})$
and $(C_{j},C_{l},K_{2},K_{2})$. This process should follow the morphism
equation: 
\begin{equation}
\tilde{G_{l}}(c_{j},c_{i})=\sum_{c_{l}}F_{l}(c_{l},c_{i})*F_{l+1}(c_{j},c_{l}),\label{eq:morph}
\end{equation}
where $\tilde{G_{l}}$ is a zero-padded version of $G_{l}$ whose
effective kernel size is $\tilde{K_{l}}=K_{1}+K_{2}-1\geq K_{l}$.
\cite{WeiWangRuiEtAl2016} showed a sufficient condition for exact
network morphism: 
\begin{equation}
\max(C_{l}C_{i}K_{1}^{2},C_{j}C_{l}K_{2}^{2})\geq C_{j}C_{i}(K_{1}+K_{2}-1)^{2}.\label{eq:exact_condition}
\end{equation}

For simplicity, we shall denote equations (\ref{eq:conv}) and (\ref{eq:morph})
as $B_{j}=G_{l}\circledast B_{i}$ and $\tilde{G}_{l}=F_{l+1}\circledast F_{l}$,
where $\circledast$ is a non-communicative multi-channel convolution
operator. We can also rewrite equation (\ref{eq:exact_condition})
as $\max(|F_{l}|,|F_{l+1}|)\geq|\tilde{G_{l}}|$, where $|*|$ measures
the size of the convolutional filter.

\subsection{Atomic Network Morphism}

\label{sub:atomic}

We start with the simplest cases. Two \emph{atomic morphing types}
are considered, as shown in Fig. \ref{fig:netmorph_atomic}: 1) a
convolutional layer is morphed into two convolutional layers (TYPE-I);
2) a convolutional layer is morphed into two-way convolutional layers
(TYPE-II). For the TYPE-I atomic morphing operation, equation (\ref{eq:morph})
is satisfied, while For TYPE-II, the filter split is set to satisfy
\begin{equation}
G_{l}=F_{l}^{1}+F_{l}^{2}.\label{eq:morph2}
\end{equation}
In addition, for TYPE-II, at the source end, the blob is split with
multiple copies; while at the sink end, the blobs are joined by addition.

\subsection{Graph Abstraction}

To simplify the representation, we introduce the following graph abstraction
for network morphism. For a convolutional neural network, we are able
to abstract it as a graph, with the blobs represented by vertices,
and convolutional layers by edges. Formally, a DCNN is represented
as a DAG $M=(V,E)$, where $V=\{B_{i}\}_{i=1}^{N}$ are blobs and
$E=\{e_{l}=(B_{i},B_{j})\}_{l=1}^{L}$are convolutional layers. Each
convolutional layer $e_{l}$ connects two blobs $B_{i}$ and $B_{j}$,
and is associated with a convolutional filter $F_{l}$. Furthermore,
in this graph, if $outdegree(B_{i})>1$, it implicitly means a split
of multiple copies; and if $indegree(B_{i})>1$, it is a joint of
addition.

Based on this abstraction, we formally introduce the following definition
for \emph{modular network morphism}: 
\begin{defn}
Let $M_{0}=(\{s,t\},e_{0})$ represent the graph with only a single
edge $e_{0}$ that connects the source vertex $s$ and sink vertex
$t$. $M=(V,E)$ represents any single-source, single-sink DAG with
the same source vertex $s$ and the same sink vertex $t$. We call
such an $M$ as a \emph{module}. If there exists a process that we
are able to morph $M_{0}$ to $M$, then we say that module $M$ is
\emph{morphable}, and the morphing process is called \emph{modular
network morphism}. 
\end{defn}
Hence, based on this abstraction, modular network morphism can be
represented as a graph transformation problem. As shown in Fig. \ref{fig:morph_process},
module (C) in Fig. \ref{fig:module_sample} can be transformed from
module $M_{0}$ by applying the illustrated network morphism operations.

For each modular network morphing, a modular network morphism equation
is associated: 
\begin{defn}
Each module essentially corresponds to a function from $s$ to $t$,
which is called a \emph{module function}. For a modular network morphism
process from $M_{0}$ to $M$, the equation that guarantees the module
function unchanged is called \emph{modular network morphism equation}. 
\end{defn}
It is obvious that equations (\ref{eq:morph}) and (\ref{eq:morph2})
are the modular network morphism equations for TYPE-I and TYPE-II
atomic morphing types. In general, the modular network morphism equation
for a module $M$ is able to be written as the sum of all convolutional
filter compositions, in which each composition is actually a path
from $s$ to $t$ in the module $M$. Let $\{(F_{p,1},F_{p,2},\cdots,F_{p,i_{p}}):\ p=1,\cdots,P,\text{ and }i_{p}\text{ is the length of path }p\}$
be the set of all such paths represented by the convolutional filters.
Then the modular network morphism equation for module $M$ can be
written as 
\begin{equation}
G_{l}=\sum_{p}F_{p,i_{p}}\circledast F_{p,i_{p}-1}\circledast\cdots\circledast F_{p,1}.
\end{equation}
As an example illustrated in Fig. \ref{fig:module_sample}, there
are four paths in module (D), and its modular network morphism equation
can be written as 
\begin{equation}
G_{l}=F_{5}\circledast F_{1}+F_{6}\circledast(F_{3}\circledast F_{1}+F_{4}\circledast F_{2})+F_{7}\circledast F_{2},
\end{equation}
where $G_{l}$ is the convolutional filter associated with $e_{0}$
in module $M_{0}$.

\subsection{The Compatibility of Network Morphism Equation}

One difficulty in this graph abstraction is in the dimensional compatibility
of convolutional filters or blobs. For example, for the TYPE-II atomic
morphing in Fig. \ref{fig:netmorph_atomic}, we have to satisfy $G_{l}=F_{l}^{1}+F_{l}^{2}$.
Suppose that $G_{l}$ and $F_{l}^{2}$ are of shape $(64,64,3,3),$
while $F_{l}^{1}$ is $(64,64,1,1)$, they are actually not addable.
Formally, we define the compatibility of modular network morphism
equation as follows: 
\begin{defn}
The modular network morphism equation for a module $M$ is \emph{compatible}
if and only if the mathematical operators between the convolutional
filters involved in this equation are well-defined. 
\end{defn}
In order to solve this compatibility problem, we need not to assume
that blobs $\{B_{i}\}$ and filters $\{F_{l}\}$ are finite dimension
tensors. Instead they are considered as infinite dimension tensors
defined with a finite support\footnote{A support of a function is defined as the set of points where the
function value is non-zero, i.e., $support(f)=\{x|f(x)\neq0\}$.}, and we call this as an extended definition. An instant advantage
when we adopt this extended definition is that we will no longer need
to differentiate $G_{l}$ and $\tilde{G_{l}}$ in equation (\ref{eq:morph}),
since $\tilde{G_{l}}$ is simply a zero-padded version of $G_{l}$. 
\begin{lem}
The operations $+$ and $\circledast$ are well-defined for the modular
network morphism equation. Namely, if $F^{1}$ and $F^{2}$ are infinite
dimension 4D tensors with finite support, let $G=F^{1}+F^{2}$ and
$H=F^{2}\circledast F^{1}$, then both $G$ and $H$ are uniquely
defined and also have finite support.\end{lem}
\begin{proof}[Sketch of Proof]
It is quite obvious that this lemma holds for the operator $+$.
For the operator $\circledast$, if we have this extended definition,
the sum in equation (\ref{eq:morph}) will become infinite over the
index $c_{l}$. It is straightforward to show that this infinite sum
converges, and also that $H$ is finitely supported with respect to
the indices $c_{j}$ and $c_{i}$. Hence $H$ has finite support. 
\end{proof}
As a corollary, we have: 
\begin{cor}
The modular network morphism equation for any module $M$ is always
compatible if the filters involved in $M$ are considered as infinite
dimension tensors with finite support. 
\end{cor}

\subsection{Simple Morphable Modules}

\begin{figure}
\begin{centering}
\subfloat[Example modules. \label{fig:module_sample}]{\begin{centering}
\includegraphics[width=0.32\linewidth]{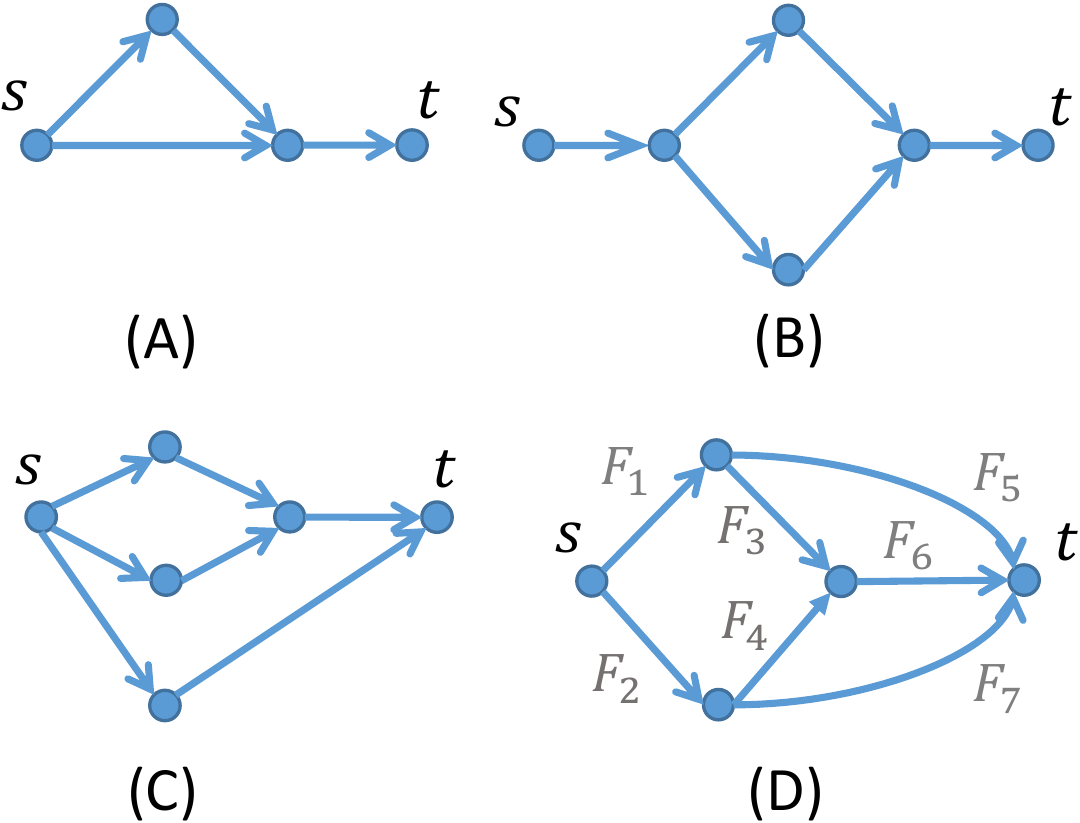} \hspace*{0.15in} 
\par\end{centering}

}\subfloat[Morphing process for module (C) and (D). \label{fig:morph_process}]{\begin{centering}
\includegraphics[width=0.57\linewidth]{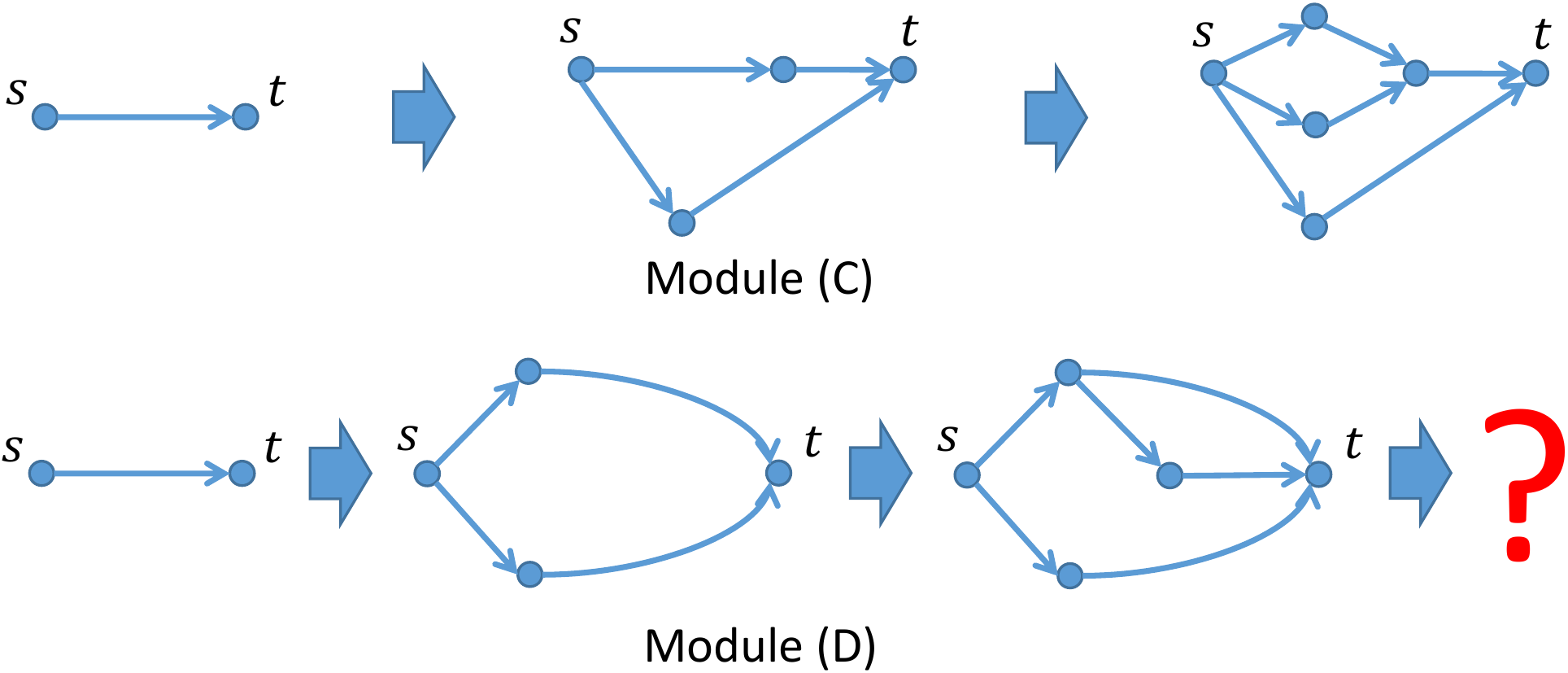} 
\par\end{centering}

}
\par\end{centering}

\caption{Example modules and morphing processes. (a) Modules (A)-(C) are simple
morphable, while (D) is not; (b) a morphing process for module (C),
while for module (D), we are not able to find such a process. }
\end{figure}

\label{sub:simple_morphable}

In this section, we introduce a large family of modules, i.e, simple
morphable modules, and then provide their morphing solutions. We first
introduce the following definition: 
\begin{defn}
A module $M$ is \emph{simple morphable} if and only if it is able
to be morphed with only combinations of atomic morphing operations. 
\end{defn}
Several example modules are shown in Fig. \ref{fig:module_sample}.
It is obvious that modules (A)-(C) are simple morphable, and the morphing
process for module (C) is also illustrated in Fig. \ref{fig:morph_process}.

For a simple morphable module $M$, we are able to identity a morphing
sequence from $M_{0}$ to $M$. The algorithm is illustrated in Algorithm
\ref{alg:simple}. The core idea is to use the reverse operations
of atomic morphing types to reduce $M$ to $M_{0}$. Hence, the morphing
process is just the reverse of the reduction process. In Algorithm
\ref{alg:simple}, we use a four-element tuple $(M,e_{1},\{e_{2},e_{3}\},type)$
to represent the process of morphing edge $e_{1}$ in module $M$
to $\{e_{2},e_{3}\}$ using TYPE-\texttt{\textsc{<type>}} atomic operation.
Two auxiliary functions \textsc{CheckTypeI} and \textsc{CheckTypeII}
are further introduced. Both of them return either \texttt{\textsc{False}}
if there is no such atomic sub-module in $M$, or a morphing tuple
$(M,e_{1},\{e_{2},e_{3}\},type)$ if there is. The algorithm of \textsc{CheckTypeI}
only needs to find a vertex satisfying $indegree(B_{i})=outdegree(B_{i})=1$,
while \textsc{CheckTypeII} looks for the matrix elements $>1$ in
the adjacent matrix representation of module $M$.

Is there a module not simple morphable? The answer is yes, and an
example is the module (D) in Fig. \ref{fig:module_sample}. A simple
try does not work as shown in Fig. \ref{fig:morph_process}. In fact,
we have the following proposition: 
\begin{prop}
Module (D) in Fig. \ref{fig:module_sample} is not simple morphable\label{prop:non_simple}.\end{prop}
\begin{proof}[Sketch of Proof]
A simple morphable module $M$ is always able to be reverted back
to $M_{0}$. However, for module (D) in Fig. \ref{fig:module_sample},
both \textsc{CheckTypeI} and \textsc{CheckTypeII} return \textsc{False}. 
\end{proof}

\subsection{Modular Network Morphism Theorem}

\begin{algorithm}[tb]
\caption{Algorithm for Simple Morphable Modules\label{alg:simple}}
\begin{algorithmic}

\State \textbf{Input:} $M_{0}$; a simple morphable module $M$

\State \textbf{Output: }The morphing sequence $Q$ that morphs $M_{0}$
to $M$ using atomic morphing operations

\State $Q=\emptyset$

\While{$M\neq M_{0}$}

\While{\textsc{CheckTypeI}($M$) is not \textsc{False}}

\State // Let $(M_{temp},e_{1},\{e_{2},e_{3}\},type)$ be the return
value of \textsc{CheckTypeI($M$)}

\State $Q.prepend((M_{temp},e_{1},\{e_{2},e_{3}\},type))$ and $M\leftarrow M_{temp}$

\EndWhile

\While{\textsc{CheckTypeII}($M$) is not \textsc{False}}

\State // Let $(M_{temp},e_{1},\{e_{2},e_{3}\},type)$ be the return
value of \textsc{CheckTypeII($M$)}

\State $Q.prepend((M_{temp},e_{1},\{e_{2},e_{3}\},type))$ and $M\leftarrow M_{temp}$

\EndWhile

\EndWhile

\end{algorithmic} 
\end{algorithm}

\begin{algorithm}[tb]
\caption{Algorithm for Irreducible Modules\label{alg:non_simple}}
\begin{algorithmic}

\State \textbf{Input:} $G_{l}$; an irreducible module $M$

\State \textbf{Output: }Convolutional filters $\{F_{i}\}_{i=1}^{n}$
of $M$

\State Initialize $\{F_{i}\}_{i=1}^{n}$ with random noise.

\State Calculate the effective kernel size of $M$, expand $G_{l}$
to $\tilde{G_{l}}$ by padding zeros.

\Repeat

\For{ $j=1$ to $n$}

\State Fix $\{F_{i}:i\neq j\}$, and calculate $F_{j}=deconv(\tilde{G_{l}},\{F_{i}:i\neq j\})$

\State Calculate loss $l=\|\tilde{G_{l}}-conv(\{F_{i}\}_{i=1}^{n})\|^{2}$

\EndFor

\Until{$l=0$ or $maxIter$ is reached}

\end{algorithmic} 
\end{algorithm}

For a module that is not simple morphable, which is called a \emph{complex
module}, we are able to apply Algorithm \ref{alg:simple} to reduce
it to an irreducible module $M$ first. For $M$, we propose Algorithm
\ref{alg:non_simple} to solve the modular network morphism equation.
The core idea of this algorithm is that, if only one convolutional
filter is allowed to change with all others fixed, the modular network
morphism equation will reduce to a linear system. The following argument
guarantees the correctness of Algorithm \ref{alg:non_simple}. 
\begin{proof}[Correctness of Algorithm \ref{alg:non_simple}]
Let $G_{l}$ and $\{F_{i}\}_{i=1}^{n}$ be the convolutional filter(s)
associated with $M_{0}$ and $M$. We further assume that one of $\{F_{i}\}$,
e.g., $F_{j}$, is larger or equal to $\tilde{G_{l}}$, where $\tilde{G_{l}}$
is the zero-padded version of $G_{l}$ (this assumption is a strong
condition in the expanding mode). The module network morphism equation
for $M$ can be written as 
\begin{equation}
\tilde{G_{l}}=C_{1}\circledast F_{j}\circledast C_{2}+C_{3},\label{eq:morph_thm}
\end{equation}
where $C_{1}$, $C_{2}$, and $C_{3}$ are composed of other filters
$\{F_{i}:i\neq j\}$. It can be checked that equation (\ref{eq:morph_thm})
is a linear system with $|\tilde{G_{l}}|$ constraints and $|F_{j}|$
free variables. Since we have $|F_{j}|\geq|\tilde{G_{l}}|$, the system
is non-deterministic and hence solvable as random matrices are rarely
inconsistent. 
\end{proof}
For a general module $M$, whether simple morphable or not, we apply
Algorithm \ref{alg:simple} to reduce $M$ to an irreducible module
$M'$, and then apply Algorithm \ref{alg:non_simple} to $M'$. Hence
we have the following theorem: 
\begin{thm}
A convolutional layer can be morphed to any module (any single-source,
single-sink DAG subnet).\label{thm:modular_netmorph} 
\end{thm}
This theorem answers the core question of network morphism, and provides
a theoretical upper bound for the capability of this learning scheme.

\subsection{Non-linearity and Batch Normalization in Modular Network Morphism}

\begin{figure}
\begin{centering}
\includegraphics[width=1\linewidth]{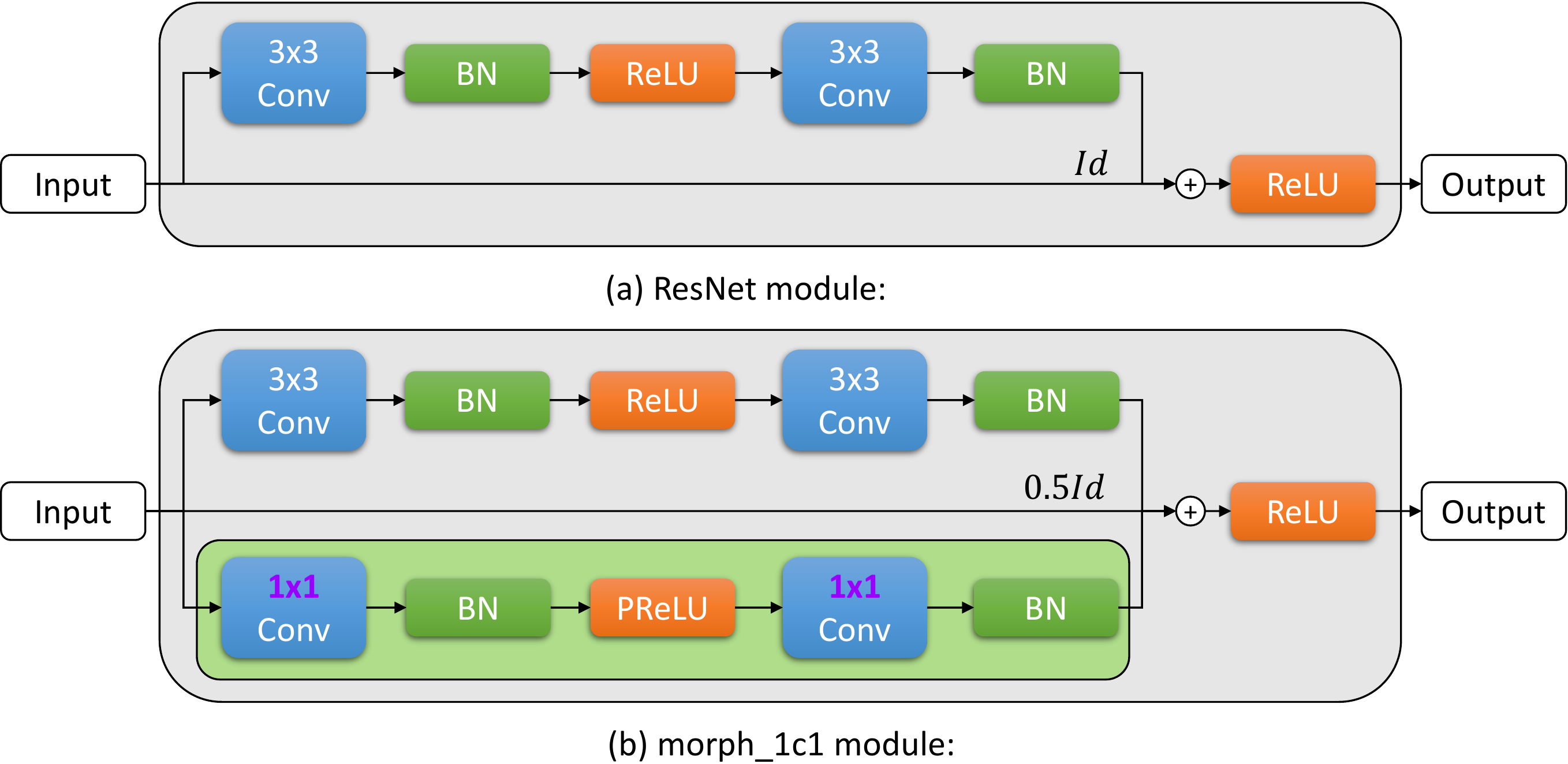} 
\par\end{centering}

\caption{Detailed architectures of the ResNet module and the \texttt{morph\_1c1}
module. \label{fig:resnet_morph}}
\end{figure}

\begin{figure}
\begin{centering}
\includegraphics[width=1\linewidth]{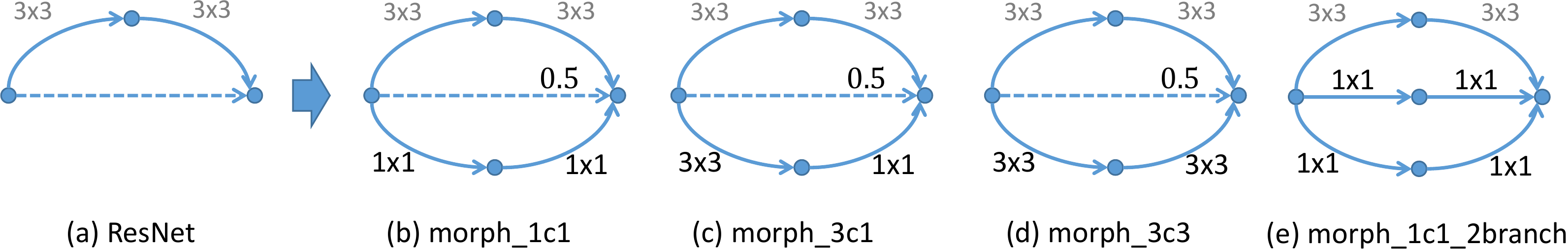} 
\par\end{centering}

\caption{Sample modules adopted in the proposed experiments. (a) and (b) are
the graph abstractions of modules illustrated in Fig. \ref{fig:resnet_morph}(a)
and (b). \label{fig:cifar_arch}}
\end{figure}

Besides the convolutional layers, a neural network module typically
also involves non-linearity layers and batch normalization layers,
as illustrated in Fig. \ref{fig:resnet_morph}. In this section, we
shall describe how do we handle these layers for modular network morphism.

For the non-linear activation layers, we adopt the solution proposed
in \cite{WeiWangRuiEtAl2016}. Instead of directly applying the non-linear
activations, we are using their parametric forms. Let $\varphi$ be
any non-linear activation function, and its parametric form is defined
to be 
\begin{equation}
P\text{-}\varphi=\{\varphi^{a}\}|_{a\in[0,1]}=\{(1-a)\cdot\varphi+a\varphi_{id}\}|_{a\in[0,1]}.
\end{equation}
The shapes of the parametric form of the non-linear activation $\varphi$
is controlled by the parameter $a$. When $a$ is initialized ($a=1$),
the parametric form is equivalent to an identity function, and when
the value of $a$ has been learned, the parametric form will become
a non-linear activation. In Fig. \ref{fig:resnet_morph}b, the non-linear
activation for the morphing process is annotated as PReLU to differentiate
itself with the other ReLU activations. In the proposed experiments,
for simplicity, all ReLUs are replaced with PReLUs.

The batch normalization layers \cite{IoffeSzegedy2015} can be represented
as 
\begin{equation}
newdata=\frac{data-mean}{\sqrt{var+eps}}\cdot gamma+beta.
\end{equation}
It is obvious that if we set $gamma=\sqrt{var+eps}$ and $beta=mean$,
then a batch normalization layer is reduced to an identity mapping
layer, and hence it can be inserted anywhere in the network. Although
it is possible to calculate the values of $gamma$ and $beta$ from
the training data, in this research, we adopt another simpler approach
by setting $gamma=1$ and $beta=0$. In fact, the value of $gamma$
can be set to any nonzero number, since the scale is then normalized
by the latter batch normalization layer (lower right one in Fig. \ref{fig:resnet_morph}b).
Mathematically and strictly speaking, when we set $gamma=0$, the
network function is actually changed. However, since the morphed filters
for the convolutional layers are roughly randomized, even though the
$mean$ of $data$ is not strictly zero, it is still approximately
zero. Plus with the fact that the data is then normalized by the latter
batch normalization layer, such small perturbation for the network
function change can be neglected. In the proposed experiments, only
statistical variances in performance are observed for the morphed
network when we adopt setting $gamma$ to zero. The reason we prefer
such an approach to using the training data is that it is easier to
implement and also yields slightly better results when we continue
to train the morphed network.

\section{Experimental Results}

\begin{table}
\caption{Experimental results of networks morphed from ResNet-20, ResNet-56,
and ResNet-110 on the CIFAR10 dataset. Results annotated with $^{\dagger}$
are from \cite{HeZhangRenEtAl2015}. \label{tab:results}}

\centering{}{\hspace*{-0.25in}\scriptsize%
\begin{tabular}{|c|c|c|c|c|c|c|c|c|}
\hline 
\multirow{2}{*}{Net Arch.} & {\scriptsize{}Intermediate } & \multirow{2}{*}{Error} & {\scriptsize{}Abs. Perf.} & {\scriptsize{}Rel. Perf.} & {\scriptsize{}\#Params.} & {\scriptsize{}\#Params.} & \multirow{2}{*}{{\scriptsize{}FLOP (million)}} & \multirow{2}{*}{{\scriptsize{}Rel. FLOP}}\tabularnewline
 & {\scriptsize{}Phases} &  & {\scriptsize{}Improv.} & {\scriptsize{}Improv.} & {\scriptsize{}(MB)} & {\scriptsize{}Rel.} &  & \tabularnewline
\hline 
\hline 
\texttt{resnet20}$^{\dagger}$ & - & 8.75\% & - & - & 1.048 & 1x & 40.8 & 1x\tabularnewline
\hline 
\texttt{morph20\_1c1} & - & 7.35\% & 1.40\% & 16.0\% & 1.138 & 1.09x & 44.0 & 1.08x\tabularnewline
\hline 
\multirow{2}{*}{\texttt{morph20\_3c1}} & - & 7.10\% & 1.65\% & 18.9\% & \multirow{2}{*}{1.466} & \multirow{2}{*}{1.40x} & \multirow{2}{*}{56.5} & \multirow{2}{*}{1.38x}\tabularnewline
\cline{2-5} 
 & \texttt{1c1} & 6.83\% & 1.92\% & 21.9\% &  &  &  & \tabularnewline
\hline 
\multirow{2}{*}{\texttt{morph20\_3c3}} & - & 6.97\% & 1.78\% & 20.3\% & \multirow{2}{*}{1.794} & \multirow{2}{*}{1.71x} & \multirow{2}{*}{69.1} & \multirow{2}{*}{1.69x}\tabularnewline
\cline{2-5} 
 & \texttt{1c1,3c1} & 6.66\% & 2.09\% & 23.9\% &  &  &  & \tabularnewline
\hline 
\multirow{2}{*}{\texttt{morph20\_1c1\_2branch}} & - & 7.26\% & 1.49\% & 17.0\% & \multirow{2}{*}{1.227} & \multirow{2}{*}{1.17x} & \multirow{2}{*}{47.1} & \multirow{2}{*}{1.15x}\tabularnewline
\cline{2-5} 
 & \texttt{1c1,half} & \textbf{6.60\%} & \textbf{2.15\%} & \textbf{24.6\%} &  &  &  & \tabularnewline
\hline 
\hline 
\texttt{resnet56}$^{\dagger}$ & - & 6.97\% & - & - & 3.289 & 1x & 125.7 & 1x\tabularnewline
\hline 
\texttt{morph56\_1c1\_half} & - & 5.68\% & 1.29\% & 18.5\% & 3.468 & 1.05x & 132.0 & 1.05x\tabularnewline
\hline 
\multirow{2}{*}{\texttt{morph56\_1c1}} & - & 5.94\% & 1.03\% & 14.8\% & \multirow{2}{*}{3.647} & \multirow{2}{*}{1.11x} & \multirow{2}{*}{138.3} & \multirow{2}{*}{1.10x}\tabularnewline
\cline{2-5} 
 & \texttt{1c1\_half} & \textbf{5.37\%} & \textbf{1.60\%} & \textbf{23.0\%} &  &  &  & \tabularnewline
\hline 
\hline 
\texttt{resnet110}$^{\dagger}$ & - & 6.61\%{\scriptsize{}$\pm$0.16} & - & - & 6.649 & 1x & 253.1 & 1x\tabularnewline
\hline 
\texttt{morph110\_1c1\_half} & - & 5.74\% & 0.87\% & 13.2\% & 7.053 & 1.06x & 267.3 & 1.06x\tabularnewline
\hline 
\multirow{2}{*}{\texttt{morph110\_1c1}} & - & 5.93\% & 0.68\% & 10.3\% & \multirow{2}{*}{7.412} & \multirow{2}{*}{1.11x} & \multirow{2}{*}{279.9} & \multirow{2}{*}{1.11x}\tabularnewline
\cline{2-5} 
 & \texttt{1c1\_half} & \textbf{5.50\%} & \textbf{1.11\%} & \textbf{16.8\%} &  &  &  & \tabularnewline
\hline 
\end{tabular}}
\end{table}

\begin{table}
\caption{Comparison results between learning from morphing and learning from
scratch for the same network architectures on the CIFAR10 dataset.
\label{tab:results_for_scratch}}

\centering{}{\small%
\begin{tabular}{|c|c|c|c|c|}
\hline 
\multirow{2}{*}{Net Arch.} & {\scriptsize{}Error} & {\scriptsize{}Error} & {\scriptsize{}Abs. Perf.} & {\scriptsize{}Rel. Perf.}\tabularnewline
 & {\scriptsize{}(scratch)} & {\scriptsize{}(morph)} & {\scriptsize{}Improv.} & {\scriptsize{}Improv.}\tabularnewline
\hline 
\hline 
\texttt{morph20\_1c1} & 8.01\% & 7.35\% & 0.66\% & 8.2\%\tabularnewline
\hline 
\texttt{morph20\_1c1\_2branch} & 7.90\% & 6.60\% & 1.30\% & 16.5\%\tabularnewline
\hline 
\texttt{morph56\_1c1} & 7.37\% & 5.37\% & 2.00\% & 27.1\%\tabularnewline
\hline 
\texttt{morph110\_1c1} & 8.16\% & 5.50\% & 2.66\% & 32.6\%\tabularnewline
\hline 
\end{tabular}}
\end{table}

In this section, we report the results of the proposed morphing algorithms
based on current state-of-the-art ResNet \cite{HeZhangRenEtAl2015},
which is the winner of 2015 ImageNet classification task.

\subsection{Network Architectures of Modular Network Morphism}

We first introduce the network architectures used in the proposed
experiments. Fig. \ref{fig:resnet_morph}a shows the module template
in the design of ResNet \cite{HeZhangRenEtAl2015}, which is actually
a simple morphable two-way module. The first path consists of two
convolutional layers, and the second path is a shortcut connection
of identity mapping. The architecture of the ResNet module can be
abstracted as the graph in Fig. \ref{fig:cifar_arch}a. For the morphed
networks, we first split the identity mapping layer in the ResNet
module into two layers with a scaling factor of 0.5. Then each of
the scaled identity mapping layers is able to be further morphed into
two convolutional layers. Fig. \ref{fig:resnet_morph}b illustrates
the case with only one scaled identity mapping layer morphed into
two convolutional layers, and its equivalent graph abstraction is
shown in Fig. \ref{fig:cifar_arch}b. To differentiate network architectures
adopted in this research, the notation \texttt{morph\_<k1>c<k2>} is
introduced, where \texttt{k1} and \texttt{k2} are kernel sizes in
the morphed network. If both of scaled identity mapping branches are
morphed, we append a suffix of `\texttt{\_2branch}'. Some examples
of morphed modules are illustrated in Fig. \ref{fig:cifar_arch}.
We also use the suffix `\texttt{\_half}' to indicate that only one
half (odd-indexed) of the modules are morphed, and the other half
are left as original ResNet modules.

\subsection{Experimental Results on the CIFAR10 Dataset}

\begin{table}
\caption{Experimental results of networks morphed from ResNet-20, ResNet-56,
and ResNet-110 on the CIFAR100 dataset. \label{tab:results-cifar100}}

\centering{}{\hspace*{-0.15in}\small%
\begin{tabular}{|c|c|c|c|c|c|c|c|c|}
\hline 
\multirow{2}{*}{Net Arch.} & {\scriptsize{}Intermediate } & \multirow{2}{*}{Error} & {\scriptsize{}Abs. Perf.} & {\scriptsize{}Rel. Perf. } & {\scriptsize{}\#Params.} & {\scriptsize{}\#Params.} & \multirow{2}{*}{{\scriptsize{}FLOP (million)}} & \multirow{2}{*}{{\scriptsize{}Rel. FLOP}}\tabularnewline
 & {\scriptsize{}Phases} &  & {\scriptsize{}Improv.} & {\scriptsize{}Improv.} & {\scriptsize{}(MB)} & {\scriptsize{}Rel.} &  & \tabularnewline
\hline 
\hline 
\texttt{resnet20} & - & 32.82\% & - & - & 1.070 & 1x & 40.8 & 1x\tabularnewline
\hline 
\texttt{morph20\_1c1} & - & 31.70\% & 1.12\% & 3.4\% & 1.160 & 1.08x & 44.0 & 1.08x\tabularnewline
\hline 
\hline 
\texttt{resnet56} & - & 29.83\% & - & - & 3.311 & 1x & 125.8 & 1x\tabularnewline
\hline 
\multirow{1}{*}{\texttt{morph56\_1c1}} & \texttt{1c1\_half} & 27.52\% & 2.31\% & \multirow{1}{*}{7.7\%} & 3.670 & 1.11x & 138.3 & \multirow{1}{*}{1.10x}\tabularnewline
\hline 
\hline 
\texttt{resnet110} & - & 28.46\% & - & - & 6.672 & 1x & 253.2 & 1x\tabularnewline
\hline 
\multirow{1}{*}{\texttt{morph110\_1c1}} & \texttt{1c1\_half} & 26.81\% & 1.65\% & \multirow{1}{*}{5.8\%} & 7.434 & 1.11x & 279.9 & \multirow{1}{*}{1.11x}\tabularnewline
\hline 
\end{tabular}}
\end{table}

\begin{table}
\caption{Comparison results between learning from morphing and learning from
scratch for the same network architectures on the CIFAR100 dataset.
\label{tab:results_for_scratch-cifar100}}

\centering{}{\small%
\begin{tabular}{|c|c|c|c|c|}
\hline 
\multirow{2}{*}{Net Arch.} & {\scriptsize{}Error} & {\scriptsize{}Error} & {\scriptsize{}Abs. Perf.} & {\scriptsize{}Rel. Perf.}\tabularnewline
 & {\scriptsize{}(scratch)} & {\scriptsize{}(morph)} & {\scriptsize{}Improv.} & {\scriptsize{}Improv.}\tabularnewline
\hline 
\hline 
\texttt{morph20\_1c1} & 33.63\% & 31.70\% & 1.93\% & 5.7\%\tabularnewline
\hline 
\texttt{morph56\_1c1} & 32.58\% & 27.52\% & 5.06\% & 15.5\%\tabularnewline
\hline 
\texttt{morph110\_1c1} & 31.94\% & 26.81\% & 5.13\% & 16.1\%\tabularnewline
\hline 
\end{tabular}}
\end{table}

\begin{figure}
\begin{centering}
\subfloat[CIFAR10]{\includegraphics[width=0.5\linewidth]{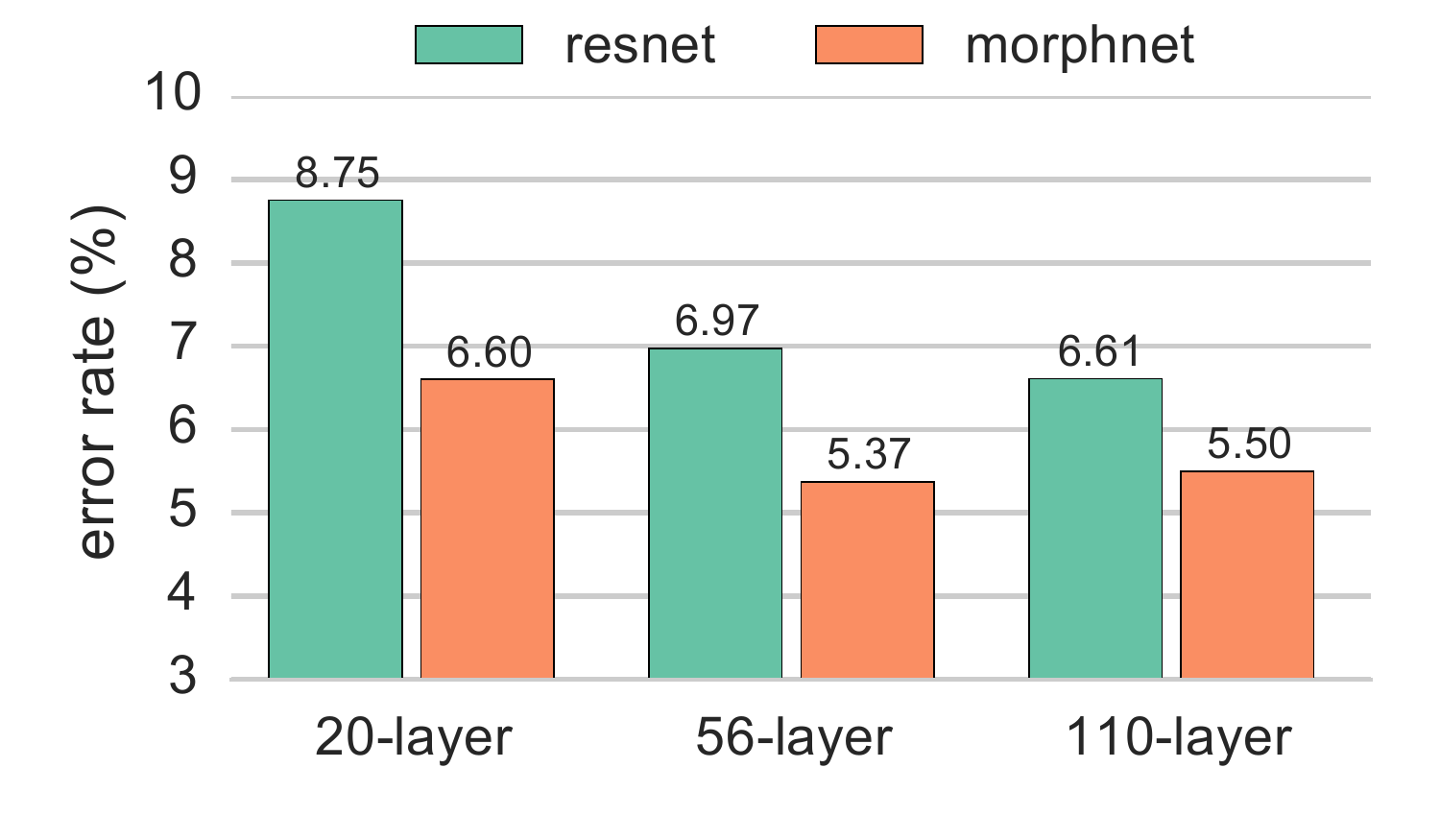}

}\subfloat[CIFAR100]{\includegraphics[width=0.5\linewidth]{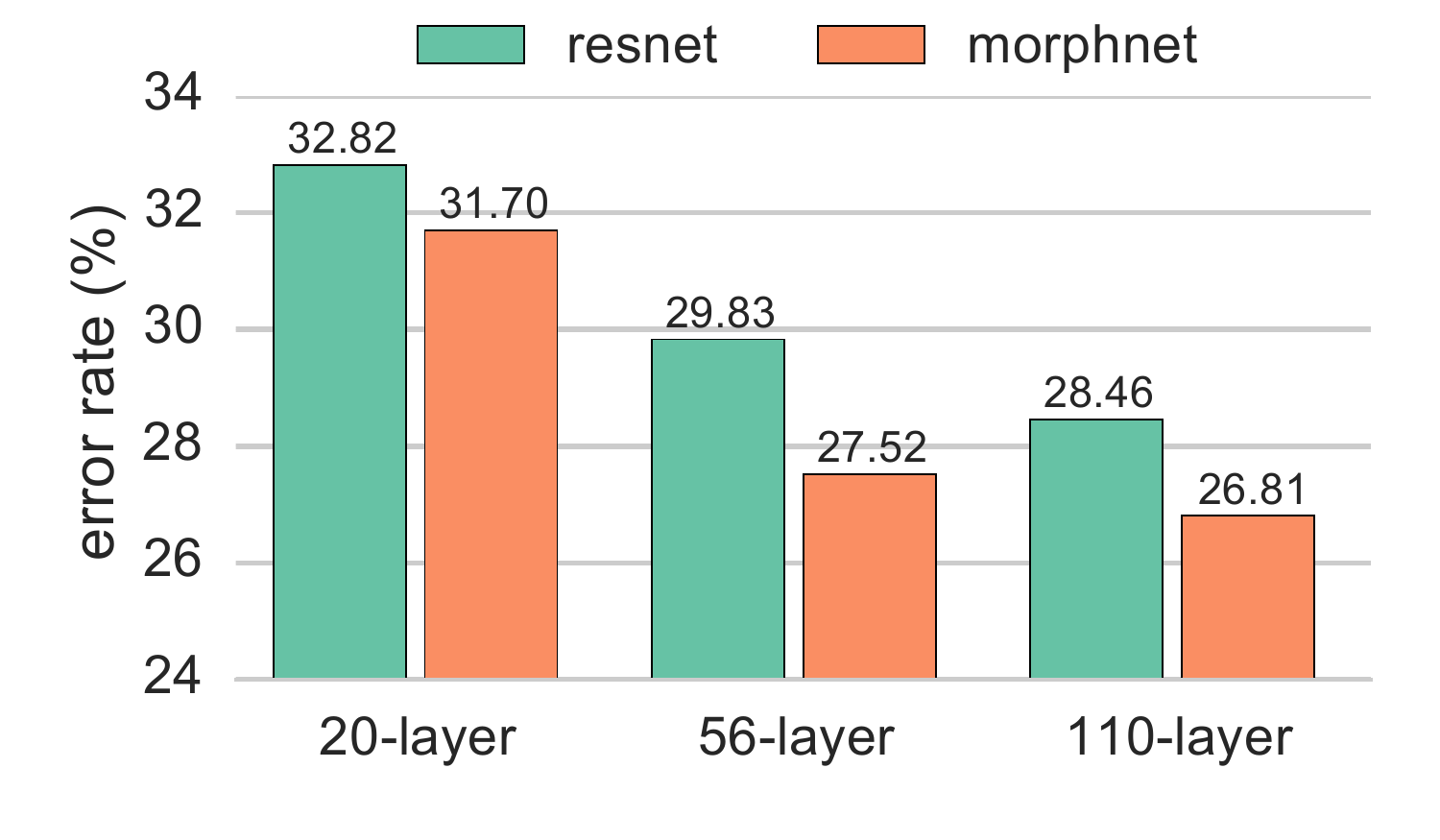}

}
\par\end{centering}

\caption{Comparison results of ResNet and morphed networks on the CIFAR10 and
CIFAR100 datasets. \label{fig:results}}
\end{figure}

CIFAR10 \cite{KrizhevskyHinton2009} is a benchmark dataset on image
classification and neural network investigation. It consists of 32$\times$32
color images in 10 categories, with 50,000 training images and 10,000
testing images. In the training process, we follow the same setup
as in \cite{HeZhangRenEtAl2015}. We use a decay of 0.0001 and a momentum
of 0.9. We adopt the simple data augmentation with a pad of 4 pixels
on each side of the original image. A 32$\times$32 view is randomly
cropped from the padded image and a random horizontal flip is optionally
applied.

Table \ref{tab:results} shows the results of different networks morphed
from ResNet \cite{HeZhangRenEtAl2015}. Notice that it is very challenging
to further improve the performance, for ResNet has already boosted
the number to a very high level. E.g., ResNet \cite{HeZhangRenEtAl2015}
made only 0.36\% performance improvement by extending the model from
56 to 110 layers (Table \ref{tab:results}). From Table \ref{tab:results}
we can see that, with only 1.2x or less computational cost, the morphed
networks achieved 2.15\%, 1.60\%, 1.11\% performance improvements
over the original ResNet-20, ResNet-56, and ResNet-110 respectively.
Notice that the relative performance improvement can be up to 25\%.
Table \ref{tab:results} also compares the number of parameters of
the original network architectures and the ones after morphing. As
can be seen, the morphed ones only have a little more parameters than
the original ones, typically less than 1.2x.

Except for large error rate reduction achieved by the morphed network,
one exciting indication from Table \ref{tab:results} is that the
morphed 20-layered network \texttt{morph20\_3c3} is able to achieve
slightly lower error rate than the 110-layered ResNet (6.60\% vs 6.61\%),
and its computational cost is actually less than $\nicefrac{1}{5}$
of the latter one. Similar results have also been observed from the
morphed 56-layered network. It is able to achieve a 5.37\% error rate,
which is even lower than those of ResNet-110 (6.61\%) and ResNet-164
(5.46\%) \cite{HeZhangRenEtAl2016}. These results are also illustrated
in Fig. \ref{fig:results}(a).

Several different architectures of the morphed networks were also
explored, as illustrated in Fig. \ref{fig:cifar_arch} and Table \ref{tab:results}.
First, when the kernel sizes were expanded from $1\times1$ to $3\times3$,
the morphed networks (\texttt{morph20\_3c1} and \texttt{morph20\_3c3})
achieved better performances. Similar results were reported in \cite{SimonyanZisserman2014}
(Table 1 for models C and D). However, because the morphed networks
almost double the computational cost, we did not adopt this approach.
Second, we also tried to morph the other scaled identity mapping layer
into two convolutional layers (\texttt{morph20\_1c1\_2branch}), the
error rate was further lowered for the 20-layered network. However,
for the 56-layered and 110-layered networks, this strategy did not
yield better results.

We also found that the morphed network learned with multiple phases
could achieve a lower error rate than that learned with single phase.
For example, the networks \texttt{morph20\_3c1} and \texttt{morph20\_3c3}
learned with intermediate phases achieved better results in Table
\ref{tab:results}. This is quite reasonable as it divides the optimization
problem into sequential phases, and thus is possible to avoid being
trapped into a local minimum to some extent. Inspired by this observation,
we then used a \texttt{1c1\_half} network as an intermediate phase
for the \texttt{morph56\_1c1} and \texttt{morph110\_1c1} networks,
and better results have been achieved.

Finally, we compared the proposed learning scheme against learning
from scratch for the networks with the same architectures. These results
are illustrated in Table \ref{tab:results_for_scratch}. As can be
seen, networks learned by morphing is able to achieve up to 2.66\%
absolute performance improvement and 32.6\% relative performance improvement
comparing against learning from scratch for the \texttt{morph110\_1c1}
network architecture. These results are quite reasonable as when networks
are learned by the proposed morphing scheme, they have already been
regularized and shall have lower probability to be trapped into a
bad-performing local minimum in the continual training process than
the learning from scratch scheme. One may also notice that, \texttt{morph110\_1c1}
actually performed worse than \texttt{resnet110} when learned from
scratch. This is because the network architecture \texttt{morph\_1c1}
is proposed for morphing, and the identity shortcut connection is
scaled with a factor of 0.5. It was also reported that residual networks
with a constant scaling factor of 0.5 actually led to a worse performance
in \cite{HeZhangRenEtAl2016} (Table 1), while this performance degradation
problem could be avoided by the proposed morphing scheme.

\subsection{Experimental Results on the CIFAR100 Dataset}

CIFAR100 \cite{KrizhevskyHinton2009} is another benchmark dataset
for tiny images that consists of 100 categories. There are 500 training
images and 100 testing images per category. The proposed experiments
on CIFAR100 follows the same setup as in the experiments on CIFAR10.
The experimental results are illustrated in Table \ref{tab:results-cifar100}
and Fig. \ref{fig:results}(b). As shown, the performance improvement
is also significant: with only around 1.1x computational cost, the
absolute performance improvement can be up to 2\% and the relative
performance improvement can be up to 8\%. For the morphed 56-layered
network, it also achieves better performance than the 110-layered
ResNet (27.52\% vs 28.46\%), and with only around one half of the
computation. Table \ref{tab:results_for_scratch-cifar100} also compares
the proposed learning scheme against learning from scratch. More than
5\% absolute performance improvement and around 16\% relative performance
improvement were achieved.

\subsection{Experimental Results on the ImageNet Dataset}

\begin{table}
\caption{Experimental results of networks morphed from ResNet-18 on the ImageNet
dataset. \label{tab:results-imagenet}}

\centering{}{\small%
\begin{tabular}{|c|c|c|c|c|c|c|}
\hline 
\multirow{2}{*}{Net Arch.} & \multirow{2}{*}{Eval. Mode} & \multirow{2}{*}{Top-1 Error} & {\scriptsize{}Abs. Perf.} & {\scriptsize{}Rel. Perf. } & \multirow{2}{*}{{\scriptsize{}FLOP (billion)}} & \multirow{2}{*}{{\scriptsize{}Rel. FLOP}}\tabularnewline
 &  &  & {\scriptsize{}Improv.} & {\scriptsize{}Improv.} &  & \tabularnewline
\hline 
\hline 
\multirow{2}{*}{\texttt{resnet18}} & 1-view & 32.56\% & - & - & \multirow{2}{*}{1.814} & \multirow{2}{*}{1x}\tabularnewline
\cline{2-5} 
 & 10-view & 30.86\% & - & - &  & \tabularnewline
\hline 
\multirow{2}{*}{\texttt{morph18\_1c1}} & 1-view & 31.69\% & 0.87\% & 2.7\% & \multirow{2}{*}{1.917} & \multirow{2}{*}{1.06x}\tabularnewline
\cline{2-5} 
 & 10-view & 29.90\% & 0.96\% & 3.1\% &  & \tabularnewline
\hline 
\multirow{2}{*}{\texttt{resnet34}} & 1-view & 29.08\% & - & - & \multirow{2}{*}{3.664} & \multirow{2}{*}{1x}\tabularnewline
\cline{2-5} 
 & 10-view & 27.32\% & - & - &  & \tabularnewline
\hline 
\multirow{2}{*}{\texttt{morph34\_1c1}} & 1-view & 27.90\% & 1.18\% & 4.1\% & \multirow{2}{*}{3.972} & \multirow{2}{*}{1.08x}\tabularnewline
\cline{2-5} 
 & 10-view & 26.20\% & 1.12\% & 4.1\% &  & \tabularnewline
\hline 
\end{tabular}}
\end{table}

\begin{wrapfigure}[15]{R}{0.65\columnwidth}%
\begin{centering}
\subfloat[18-layer.]{\begin{centering}
\includegraphics[width=0.5\linewidth]{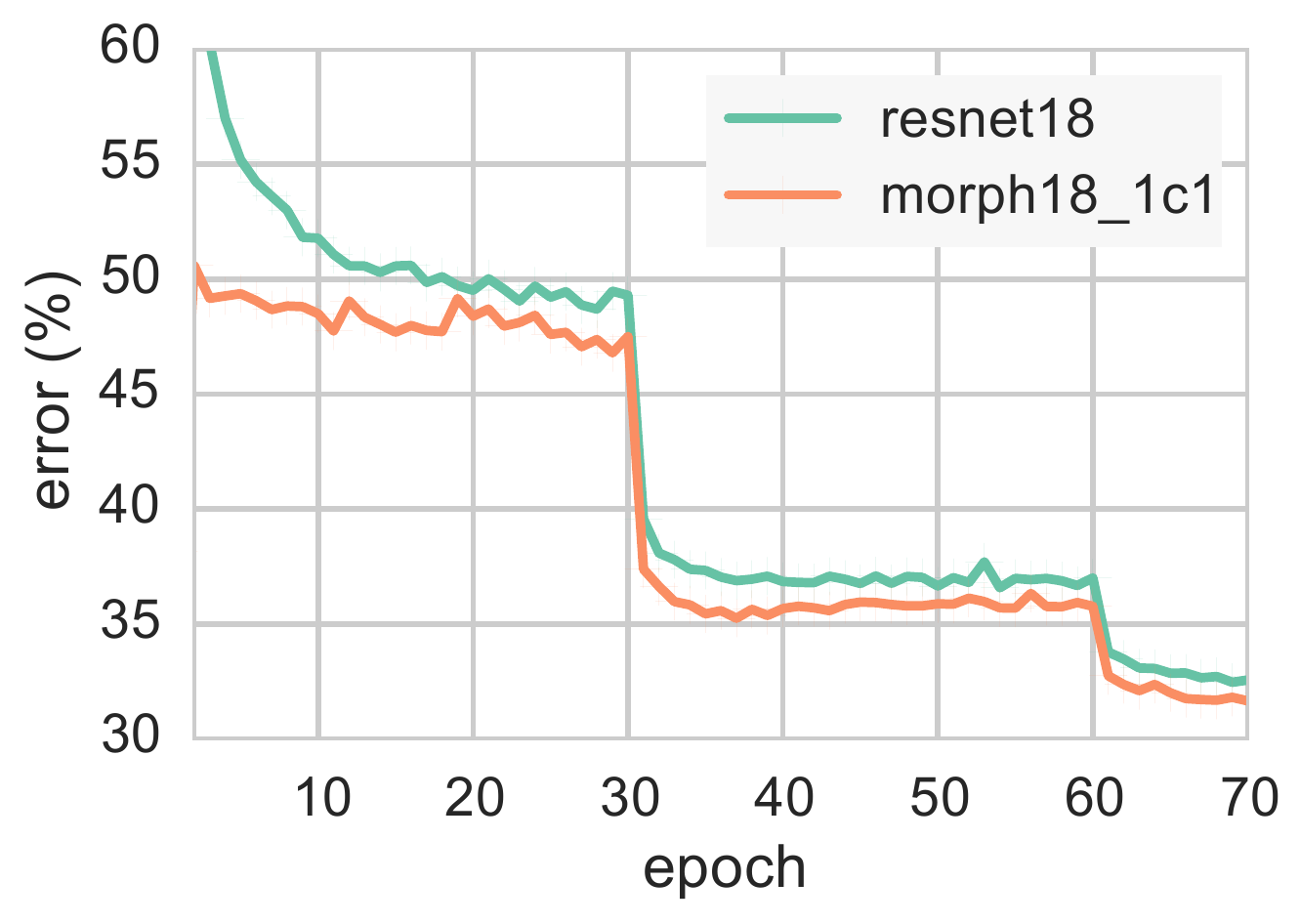}
\par\end{centering}

}\subfloat[34-layer.]{\begin{centering}
\includegraphics[width=0.5\linewidth]{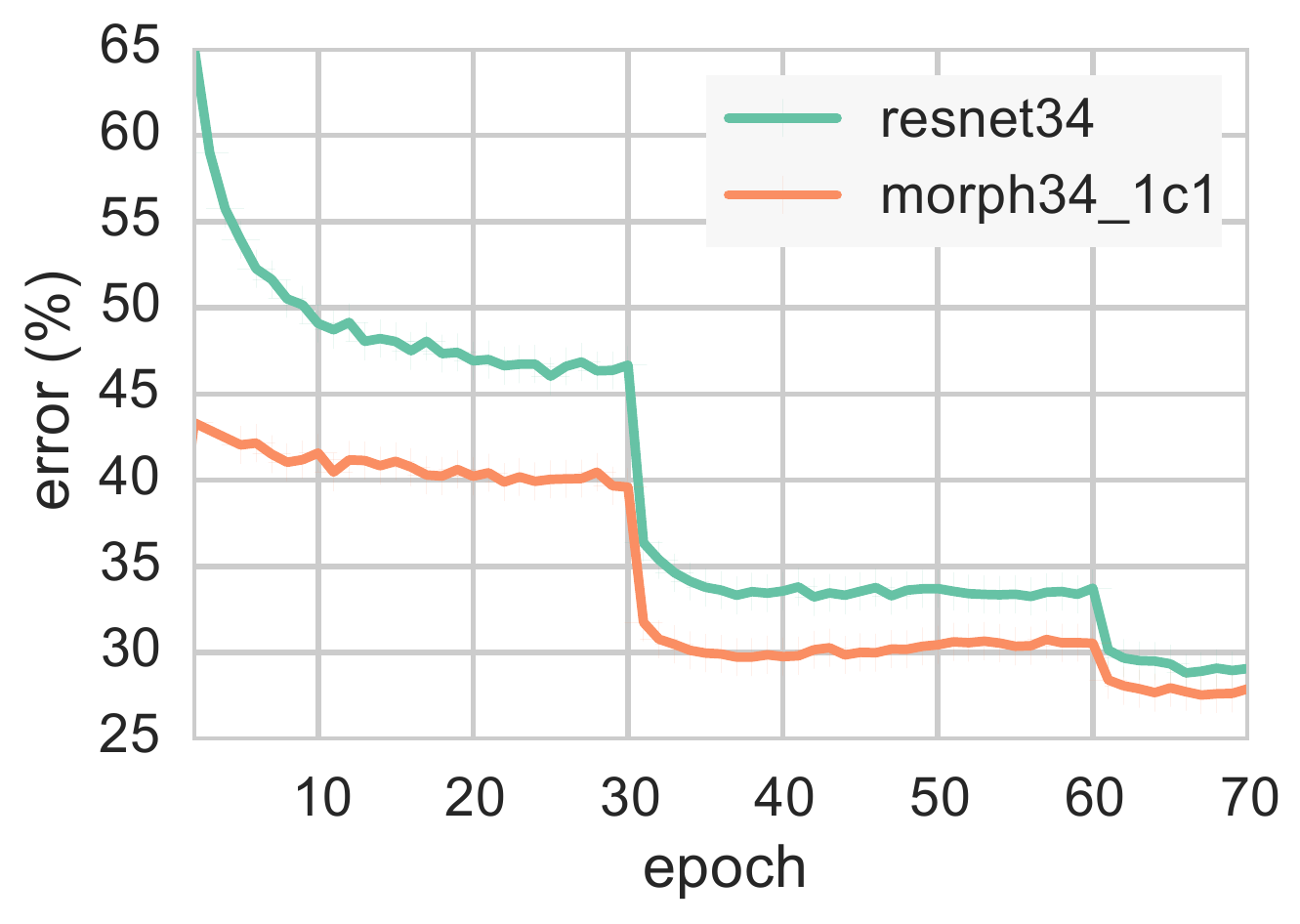}
\par\end{centering}

}
\par\end{centering}

\caption{Evaluation errors on the ImageNet dataset. \label{fig:results_imagenet}}
\end{wrapfigure}%

We also evaluate the proposed scheme on the ImageNet dataset \cite{RussakovskyDengSuEtAl2014}.
This dataset consists of 1,000 object categories, with 1.28 million
training images and 50K validation images. For the training process,
we use a decay of 0.0001 and a momentum of 0.9. The image is resized
to guarantee its shorter edge is randomly sampled from {[}256,480{]}
for scale augmentation. A $224\times224$ patch or its horizontal
flip is randomly cropped from the resized image, with the image data
per-channel normalized. We train the networks using SGD with a batch
size of 256. The learning rate starts from 0.1 and is decreased with
a factor of 0.1 for every 30 epochs. The networks are trained for
a total of 70 epochs.

The comparison results of the morphed and original ResNets for both
18-layer and 34-layer networks are illustrated in Table \ref{tab:results-imagenet}
and Fig. \ref{fig:results_imagenet}. As shown in Table \ref{tab:results-imagenet},
\texttt{morph18\_1c1} and \texttt{morph34\_1c1} are able to achieve
lower error rates than ResNet-18 and ResNet-34 respectively, and the
absolute performance improvements can be up to 1.2\%. We also draw
the evaluation error curves in Fig. \ref{fig:results_imagenet}, which
shows that the morphed networks \texttt{morph18\_1c1} and \texttt{morph34\_1c1}
are much more effective than the original ResNet-18 and ResNet-34
respectively.

\section{Conclusions}

This paper presented a systematic study on the problem of network
morphism at a higher level, and tried to answer the central question
of such learning scheme, i.e., whether and how a convolutional layer
can be morphed into an arbitrary module. To facilitate the study,
we abstracted a modular network as a graph, and formulated the process
of network morphism as a graph transformation process. Based on this
formulation, both simple morphable modules and complex modules have
been defined and corresponding morphing algorithms have been proposed.
We have shown that a convolutional layer can be morphed into any module
of a network. We have also carried out experiments to illustrate how
to achieve a better performing model based on the state-of-the-art
ResNet with minimal extra computational cost on benchmark datasets.
The experimental results have demonstrated the effectiveness of the
proposed morphing approach.

\bibliographystyle{iclr2017_conference}
\bibliography{dcnn}

\end{document}